\documentclass[letterpaper, 10pt, conference, final]{ieeeconf}
\IEEEoverridecommandlockouts{}
\usepackage{graphicx} 
\usepackage[dvipsnames]{xcolor}
\usepackage{amsfonts,amsmath,amssymb}
\usepackage[utf8]{inputenc}
\usepackage{epstopdf}
\usepackage{soul}
\usepackage{url}
\usepackage{todonotes}
\usepackage{float}
\usepackage{hyperref}
\usepackage[font=small]{caption, subcaption}

\newtheorem{theorem}{Theorem}[section]

\newtheorem{proposition}[theorem]{Proposition}

\newcommand{\rr}{\mathbf{r}}
\newcommand{\cc}{\mathbf{c}}

\newcommand{\R}{\mathbb{R}}
\newcommand{\Rr}{\mathcal{R}}

\newcommand{\E}{\mathcal{E}}
\newcommand{\T}{\mathcal{T}}

\newcommand{\N}{\mathcal{N}}
\newcommand{\CP}{\mathcal{CP}}

\newcommand{\figref}[1]{Fig.~\ref{#1}}
\newcommand{\norm}[1]{\left\lVert#1\right\rVert}

\definecolor{forest}{rgb}{0.0, 0.5, 0.0}

\title{\LARGE \bf Direct Bézier-Based Trajectory Planner for \\ Improved Local Exploration of Unknown Environments}
\author{Lorenzo Gentilini$^*$, Dario Mengoli$^*$, and Lorenzo Marconi$^*$
  \thanks{$^*$L. Gentilini, D. Mengoli and L. Marconi are with the Center for Research on Complex Automated Systems (CASY), Department of Electrical, 
  Electronic and Information Engineering (DEI), University of Bologna, Bologna, Italy (e-mails:
  \texttt{\{lorenzo.gentilini6, dario.mengoli2, lorenzo.marconi\}@unibo.it}).}}

\begin{document}
\maketitle
\thispagestyle{empty}
\pagestyle{empty}

\begin{abstract}
Autonomous exploration is an essential capability for mobile robots, as the majority of their applications require the ability to efficiently collect information about their surroundings. 
In the literature, there are several approaches, ranging from frontier-based methods to hybrid solutions involving
the ability to plan both local and global exploring paths, but only few of them focus on improving local exploration
by properly tuning the planned trajectory, often leading to ``stop-and-go'' like behaviors.
In this work we propose a novel RRT-inspired Bézier-based next-best-view trajectory planner able to deal with the problem of \textit{fast local exploration}. 
Gaussian process inference is used to guarantee fast exploration gain retrieval while still being consistent with the exploration task.
The proposed approach is compared with other available state-of-the-art algorithms and tested in a real-world scenario. 
The implemented code is publicly released as open-source code to encourage further developments and benchmarking.
\end{abstract}

\section{Introduction}\label{sec:INTRODUCTION}
The ability to autonomously plan and execute informative trajectories in previously unknown environments is a fundamental
requirement for mobile robots. As a matter of fact, they started to be employed in a huge number of different applications
which require the ability to efficiently collect new information about the surroundings, such as 
surface inspection, object search, weed recognition, search and rescue missions, and others more.
The problem of Informative Path Planning (IPP), jointly with the problem of environment exploration, has been extensively
studied in literature and a rich variety of approaches have been proposed so far.
The majority of recent works had focused on novel hybrid approaches leveraging on the interplay between the concepts of
\textit{local} and \textit{global} exploration~\cite{selin2019efficient,schmid2020efficient}. 
In particular, the major issue behind such works is related to how efficiently combine the two local and global exploration steps,
and how to plan high informative global paths out of the current environment information. A limited number of
papers focused on improving the local exploration step~\cite{selin2019efficient}.
In this work we aim to push the current state-of-the-art toward a more smooth and resilient solutions for
local exploration in cluttered and possibly varying environments.
In particular, we focus on the problem of \textit{fast exploration} using Unmanned Aerial Vehicles (UAVs).
The high maneuverability of the adopted robot motivates solutions able to stress the quadrotor to fully exploit both its computational and dynamical capabilities.
\begin{figure}[!t]
	\centering
	\includegraphics[trim={20cm 10cm 23cm 10cm}, clip = true, scale=.15]{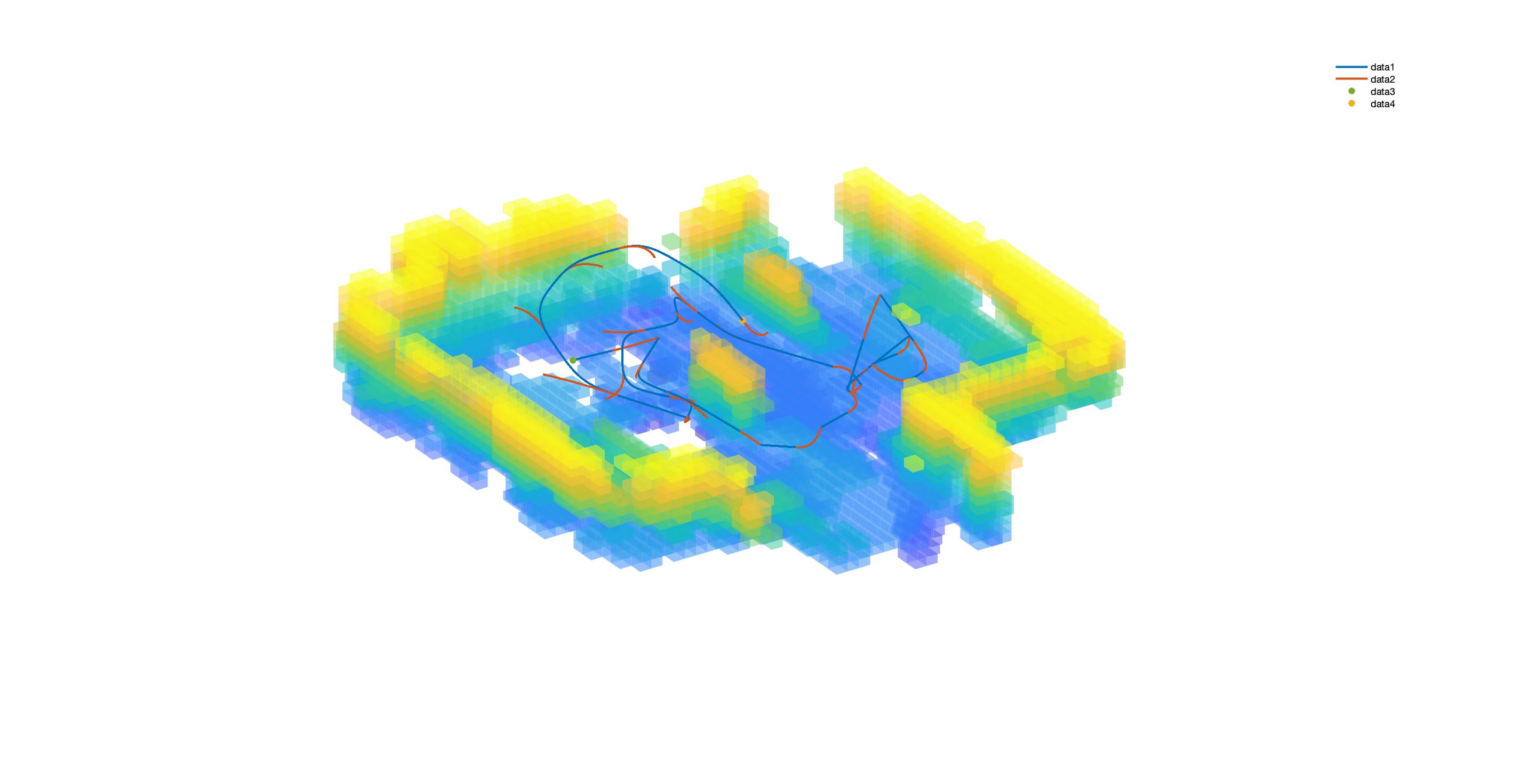}
	\vspace{-0.6cm}
	\caption{Qualitative evaluation of the proposed method in a real-world experiment.
			The Bézier-based exploration succeeded in fast planning motion inside the unknown area
			without forcing zero end velocities and successfully avoiding the two obstacles
			placed at the center. In the figure, blue lines represent the reference trajectory, while in red
			are depicted the planned \textit{safe} maneuvers.}\label{fig:REAL-SCENARIO-RESULTS}
\end{figure}
The core of this work consists of a new Rapid-exploring Random Tree (RRT) inspired sampling-based exploration algorithm that aims to directly
plan high informative feasible trajectories in known space, leading to an optimal local exploration procedure.
We show that the combination of the planning of both path and the associated timing law leads to a solution outperforming the state-of-the-art approaches
in this field, which usually plans point-to-point trajectories requiring stopping the robot at each exploration step.
In particular, we compared the proposed approach against one of the most relevant state-of-the-art solutions represented by~\cite{selin2019efficient}.
The core contributions of this work can be summarised as follows:
\begin{enumerate}
	\item We propose a new RRT-inspired Bézier-based local trajectory planner suitable for fast trajectory planning in a known environment avoiding inefficient ``stop-and-go'' like behaviors.
	The proposed approach is conceived to be easily extended and completed with a state-of-the-art global planning routine.
	\item We provide a novel and less conservative condition to guarantee the non-collision property of the generated path.
		  Such a condition is meant to be a natural extension of the well-known Bézier curve properties.
	\item Gaussian process inference is used to allow fast reconstruction gain retrieval, while its full computation
		  is left as a background thread. This guarantees fast trajectory planning while preserving consistency with the exploration task.
	\item We extensively tested the proposed solution both in simulation and in real-world scenarios. Furthermore, we released the code as an open-source ROS package\footnote{\href{https://github.com/casy-lab/BezierFastExploration.git}{github.com/casy-lab/BezierFastExploration.git}}
		  to encourage further developments and benchmarking.
\end{enumerate}
\section{Related Works}\label{sec:RELATED-WORKS}
Although the number of solutions presented in the literature is quite variegated, the majority of them can be classified as \textit{frontier-based} or \textit{sampling-based} methods.
The former class was pioneered in~\cite{yamauchi1997frontier} and  later more comprehensively developed in~\cite{julia2012comparison}. The key idea is to guide the agent toward the borders between free and unmapped space (aka frontiers),
since these points may represent those with higher potential information gain.
Exploration is then carried out by extracting the map frontiers and by navigating through them sequentially.
Several works propose to extend it by adding constraints to ensure low localization errors~\cite{stachniss2005information}.
The basic frontier-based approach has been also extended to high-speed flight for fast exploration in~\cite{cieslewski2017rapid}.
In this case, the authors propose to extract frontiers only inside the current Field-Of-View (FoV)
and select the one leading to the minimal change in velocity.
In recent years, other works focused on rapid exploration~\cite{zhou2021fuel}, by planning global coverage paths and optimising
them with respect to the robot dynamics, and on the reformulation of the frontier information gain as a differentiable function~\cite{deng2020robotic}, allowing paths to be optimised with gradient information.
On the other hand, \textit{sampling-based} methods typically sample random viewpoints to explore the space in a Next-Best-View (NBV)
fashion~\cite{connolly1985determination, maver1993occlusions}. Much of the work in this domain can be traced back to~\cite{gonzalez2002navigation}, where the NBV problem has been moved for the first time from the computer graphic field to the robotic domain,
with the introduction of the notion of reconstruction gain.
The concept of NBV exploration has been afterward extended by~\cite{bircher2016receding}, where the building of a RRT
allows one to weight both the amount of information gained at the viewpoints, and during the agent motion to reach each viewpoint.
Unlike frontier-based methods, which are difficult to adapt to other tasks, the sampling-based ones have the advantage to allow any kind of gain formulation.
Thanks to that, the original NBV algorithm was extended to consider the uncertainty of localisation~\cite{papachristos2017uncertainty, tovar2006planning}
and the visual importance of different objects~\cite{dang2018visual}.
On the other hand, sampling-based methods suffer from stacking in local minima, leading to a premature ending of the exploration
procedure in unlucky scenarios. For this reason, the recent trend is to merge the two frontiers and 
sampling-based approaches in a local-global exploration fashion.
The pioneer of this idea was~\cite{charrow2015information}, that utilises a frontier method to detect global goals and
supplements these with motion primitives for local exploration.
More recent approaches, instead, leverage the capabilities of sampling-based methods and employ additional planning
stages to escape from local minima~\cite{selin2019efficient,dang2019graph}.
Other approaches focus on memorising previously visited places and sampled information under the format of roadmaps~\cite{witting2018history, xu2021autonomous}.
Similarly, the work presented in~\cite{schmid2020efficient} continuously maintains and expands a single RRT of candidate paths.
Although the literature has seen some impressive works in the field of NBV, there are very few works concentrating on fast exploration.
Even if previous solutions are able to quickly plan globally coverage paths~\cite{schmid2020efficient, xu2021autonomous, kompis2021informed},
the problem of efficient trajectory planning is rarely addressed.
Recently,~\cite{dharmadhikari2020motion} proposes a novel primitive-based algorithm that addresses exactly this problem,
introducing a new planning paradigm where both path and timing law are allocated directly at the planning stage.
\section{Proposed Approach}\label{sec:PROPOSED-APPROACH}
In order to deal with the exploration problem, we employ a sampling-based receding-horizon trajectory planning approach.
Similarly to~\cite{selin2019efficient} and~\cite{bircher2016receding}, the central idea of the proposed algorithm is to expand
a RRT~\cite{lavalle1998rapidly} by iterative sampling new candidate viewpoints.
The obtained RRT is executed one node at a time, in a receding-horizon fashion.
Unlike previous works, our algorithm employs a Bézier curve parameterisation to grow and maintain a tree of possible trajectory segments.
The proposed approach weights both potential information gain and trajectory cost during the selection of the next goal.
Moreover, the planned trajectory does not constrain the end velocity to zero, thus the exploration can be performed quickly
by avoiding ``stop-and-go'' like behaviors.
The proposed algorithm is thought to be applied to Unmanned Aerial Vehicles (UAVs), by means of a quadrotor in our case, in order to perform optimal \textit{local} exploration steps inside a 3D volume, without keeping in consideration the possibility to plan global trajectories.
Motivated by the promising results obtained using hybrid approaches~\cite{selin2019efficient}, we allow adaptability of the proposed solution by letting the tree
to be easily extended with global exploration routines. In particular, we implemented efficient rewiring procedures in order to keep in memory and continuously
refine the same tree, following the ideas of roadmaps memorisation~\cite{xu2021autonomous} and continuous tree expansion~\cite{schmid2020efficient}.

\subsection{Bézier Trajectory Parameterization}\label{sec:TRAJECTORY-PARAMETERIZATION}
In this work, instead of using traditional polynomial functions, we adopt the Bernstein polynomial basis and define trajectories as piecewise Bézier curves.
A Bézier curve is completely defined by its degree $n$ and a set of $n+1$ control points $\CP=[q_0 \cdots q_{n}]$, with $q_i \in \R$.
The curve can be evaluated, for any $ u \in [0, 1]$, as
\begin{equation}
	\label{eq:BEZIER}
	q(u) = \sum_{i=0}^{n} B_i^{n}(u) q_i,
\end{equation}
where the basis functions $B_i^{n}(u)$ are $n$-th degree \textit{Bernstein basis polynomials}~\cite{farouki2012bernstein, biagiotti2008trajectory} of the form
\begin{equation*}
	B_i^n(u) = \frac{n!}{i!(n-i)!} u^i {(1-u)}^{n-i}.
\end{equation*}
The aforementioned polynomials enjoy a partition-of-unity property (i.e. $\sum_{i=0}^{n} B_i^{n}(u) = 1$ for all $u$),
by which the curve defined by Equation~\eqref{eq:BEZIER} is constrained inside the convex hull generated by its control points $\CP$.
Moreover, a $n$-degree Bézier curve is always $n$ times differentiable and its derivatives preserve a Bézier structure of lower degree.
In particular, $q'(u) := dq/du$ is a Bézier curve of order $n-1$ whose control points $\CP'$ can be evaluated as
$q'_i = n(q_{i+1}-q_i)$ $\forall i=0,\dots,n-1.$
The overall quadrotor reference trajectory can be expressed through the evolution of its \textit{flat outputs}~\cite{mellinger2011minimum},
$\sigma = {\left[ \rr, \phi \right]}^T$, where $\rr={\left[x, y, z \right]}^T \in \R^3$ represents the coordinates of the
center of mass in the world coordinate system, while $\phi \in \R$ is the yaw angle.
Both the quantities $\rr$ and $\phi$ are expressed as $m$-segment piecewise Bézier curves of order $n_r$ and $n_{\phi}$, respectively
\begin{equation*}
	\rr(t)=
	\begin{cases}
		\sum_{i=0}^{n_r} B_i^{n_r}(\delta_1) \rr^1_i & \hspace{-0.2cm} t \in [T_0, T_1],\\
		\sum_{i=0}^{n_r} B_i^{n_r}(\delta_2) \rr^2_i & \hspace{-0.2cm} t \in [T_1, T_2],\\
		\hspace{0.2cm} \vdots & \vdots \\
		\sum_{i=0}^{n_r} B_i^{n_r}(\delta_m) \rr^m_i & \hspace{-0.2cm} t \in [T_{m-1}, T_{m}]\\
	\end{cases}
\end{equation*}
with $\tau_i = \frac{t-T_{i-1}}{T_{i}-T_{i-1}}$. The same definition holds also for $\phi(t)$, with $n_r$ and $\rr^j_i$ substituted by $n_{\phi}$ and $\phi_{i}^j$.
The quantities $\rr_i^j \in \R^3$ and $\phi_i^j \in \R$
describe the $i^{th}$ control point of the $j^{th}$ trajectory segment of $\rr(t)$ and $\phi(t)$ respectively,
while $T_{j-1}$ and $ T_j$ are the start and end time of the $j^{th}$ trajectory segment.
Note that the introduced time scaling does not affect the spatial path described by the Bézier curve, but strongly affects its derivatives  as 
\begin{align*}
	{\rr'}^j_i & = \frac{n_r(\rr^j_{i+1}-\rr^j_i)}{T_{j-1}-T_j} \hspace{0.2cm} \forall i=0,\dots,n_r-1 , \\
	{\phi'}^j_i & = \frac{{n_{\phi}}(\phi^j_{i+1}-\phi^j_i)}{T_{j-1}-T_j} \hspace{0.2cm} \forall i=0,\dots,n_{\phi}-1 .
\end{align*}
\begin{figure}[!t]
	\begin{minipage}{.45\linewidth}
		\centering
		\subfloat[]{\vspace{-0.8cm}\label{fig:BEZIER-CURVE-A}\includegraphics[scale=.6]{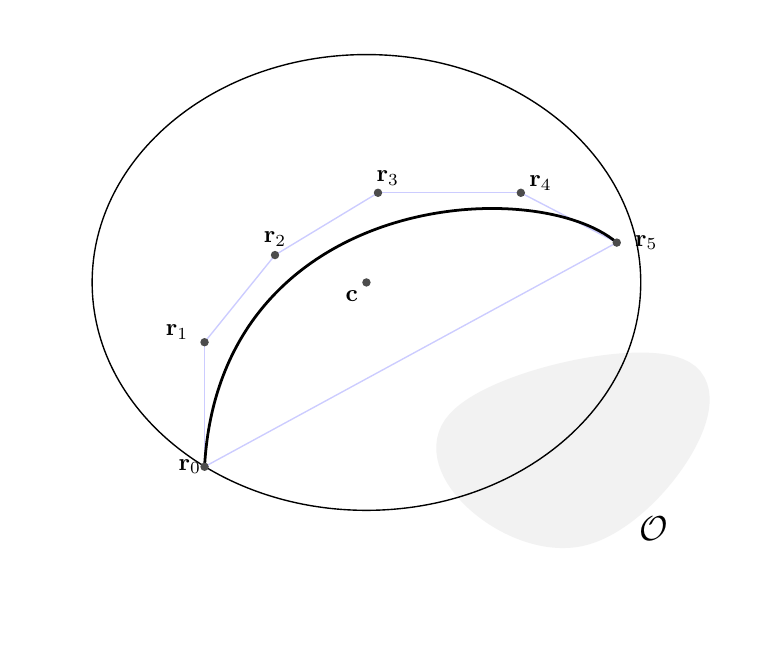}}
	\end{minipage}
	\begin{minipage}{.45\linewidth}
		\centering
		\subfloat[]{\vspace{-0.8cm}\label{fig:BEZIER-CURVE-B}\includegraphics[scale=.6]{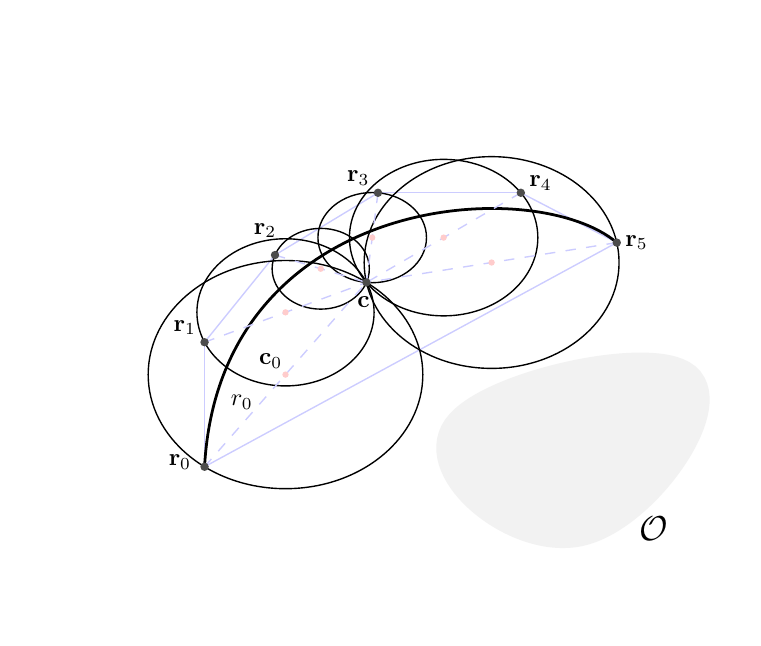}}
	\end{minipage}
	\caption{Representation of a fifth-degree Bézier curve with
	(a) the classical sphere used for collision checking~\cite{tang2020real}
	(b) the multiple spheres envelope used in this work.
	The $\mathcal{O}$ shaded gray area represents a generic obstacle.}\label{fig:BEZIER-CURVE}
\end{figure}
The \textit{convex hull containment} property is a powerful tool to verify both the trajectory feasibility
in terms of dynamic constraints, such as velocity or acceleration bounds, and to check for collisions.
\figref{fig:BEZIER-CURVE-A} reports the classical condition used for collision checking with Bézier curves~\cite{tang2020real},
where the overall curve is constrained inside a \textit{safe} sphere.
The aforementioned approach often results in being too conservative, as a matter of fact the considered sphere
is far to be tight over the convex hull, and thus over the curve itself.
For this reason we formulate a new proposition that represents a less conservative tool
to verify collision (see \figref{fig:BEZIER-CURVE-B}).
\begin{proposition}\label{th:ENVELOPE-CONTAINMENT}
	Let $\rr(u)$ be a Bézier curve of order $n$, with control points $\CP= [\rr_0, \dots, \rr_n]$.
	Moreover, let $r_i \in \R$  and $\cc_i \in \R^3$ with $i=0, \dots, n$ be respectively the radii
	and centre of $n$ spheres ($\mathcal{C}_0 \dots \mathcal{C}_{n}$),
	defined as:
	\begin{equation*}
		\begin{split}
			\cc_i & = (\rr_i + \cc)/2, \\
			r_i & = \| \rr_i - \cc_i \|,
		\end{split}
	\end{equation*}
	with $\cc$ be the centre of the convex hull generated by $\CP$, i.e. $\cc = \sum_{i=0}^{n} \rr_i/n$.
	Then the curve $\rr(u)$ is entirely contained inside the spheres envelope, namely
	\begin{equation*}
		\rr(u) \in \bigcup_{i=1}^{n} \mathcal{C}_i \hspace{0.5cm} \forall u \in [0,1].
	\end{equation*}
\end{proposition}
\begin{proof}
    The proof follows from the fact that $\cc$ belongs to the convex hull and that the spheres envelope composed by
	$\mathcal{C}_i$ and $\mathcal{C}_{i+1}$ always contains the convex hull edge $\overline{\rr_i\rr_{i+1}}$.
	The first statement is true by construction, since $\cc$ is a linear combination of $\rr_i$, while the second one
	follows from the triangle inequality $\| \rr_i - \rr_{i+1} \| \le \| \rr_i - \cc \| + \| \rr_{i+1} - \cc \|$. 
\end{proof}

The aforementioned proposition states that the convex hull containment property can be reformulated taking into account
a set of $n$ spheres. Since this set of spheres results to be tighter around the curve with respect to a single big safe ball,
the use of this proposition in formulating a new collision condition results in a less conservative approach.
The following proposition states the sufficient condition for non-collision as a corollary of Proposition~\ref{th:ENVELOPE-CONTAINMENT}.
\begin{proposition}\label{th:COLLISION-FREE}
	Let $\rr(u)$ be a Bézier curve of order $n$, with control points $\CP= [\rr_0, \dots, \rr_n]$.
	Moreover, let $\cc_i$ and $r_i$ with $i=0,\dots,n$ be the centre and radii of $n$ spheres defined as in Proposition~\ref{th:ENVELOPE-CONTAINMENT}.
	The curve $\rr(u)$ is said to be collision-free, with a safety bound of $d^{\text{safe}} \in \R_{+}$, if
	the condition $r_i - d^{\text{obs}}_{\cc_i} - d^{\text{safe}} > 0$
	holds $ \forall i = 0, \dots, n$, where $d^{\text{obs}}_{\cc_i}$ represents the Euclidean distance of $\cc_i$ from the closest obstacle.
\end{proposition}

From now on, we use fifth-order Bézier curves to represent the quadrotor position ($n_r=5$),
while the yaw trajectory is parameterised using third-order Bézier curves ($n_{\phi}=3$).

\subsection{Tree Structure}\label{sec:TREE-STRUCTURE}
The proposed algorithm works by growing and maintaining, at each iteration, a tree $\T = (\N, \E)$ of possible trajectories.
Such tree consists of a set of nodes $\N = \{N_1, \dots, N_{n_n}\}$ and a set of edges $\E= \{ E_i, \dots , E_{n_e} \}$.
Each node $N_i$ is completely defined by the following five quantities
\begin{equation*}
	N_i = \left\{ g_i, c_i, \delta_i, \CP^{\rr}_i, \CP^{\phi}_i \right\}
\end{equation*}
where $g_i = g(N_i)$ represents the amount of information gained if that node is executed,
and $c_i = c(N_i)$ is the cost associated to the node execution. $\CP^{\rr}_i$ and $\CP^{\phi}_i$
are the two sets of control points defining the trajectories $\rr_i(t)$ and $\phi_i(t)$, while $\delta_i$ is the execution time.
Two nodes $N_{i-1}$ and $N_{i}$ are connected by an edge $E_{i-1}$ only if the first $(n_{\rr/\phi}+1)/2$
control points of the latter node satisfy some continuity criterion with the last $(n_{\rr/\phi}+1)/2$ control points of the former one.
This constraint is required to ensure continuity among all trajectory segments of the tree.
In particular, since $n_{\rr} = 5$ and $n_{\phi} = 3$, we enforce continuity up to the third derivative along
$\rr(t)$ and continuity up to the second derivative along $\phi(t)$, namely
\begin{align}
	\label{eq:FIRST-COND}
	&\rr^{i}_0  = \rr^{i-1}_5,\\
	\label{eq:SECOND-COND}
	&\frac{1}{\delta_{i}}(\rr^{i}_1 - \rr^{i}_0)  = \frac{1}{\delta_{i-1}}(\rr^{i-1}_5 - \rr^{i-1}_4),\\
	\label{eq:THIRD-COND}
	&\frac{1}{\delta_{i}^2}(\rr^{i}_2 - 2\rr^{i}_1 + \rr^{i}_0)  = \frac{1}{\delta_{i-1}^2}(\rr^{i-1}_5 - 2\rr^{i-1}_4 + \rr^{i-1}_3),\\
	\label{eq:FOURTH-COND}
	&\phi^{i}_0  = \phi^{i-1}_3,\\
	\label{eq:FIFTH-COND}
	&\frac{1}{\delta_{i}}(\phi^{i}_1 - \phi^{i}_0)  = \frac{1}{\delta_{i-1}}(\phi^{i-1}_3 - \phi^{i-1}_2).
\end{align}
The aim is to plan sub-optimal trajectories by maximising a user-specified utility function $J(\Rr(N_i))$,
with $\Rr(N_i)$ be the sequence of nodes connecting $N_i$ to the tree root, which properly combines gains
and costs of all nodes in $\Rr(N_i)$. In this context, the tree root is defined as the tree node which is
about to be executed by the flying agent.
It results that the agent behavior strongly depends on the choice of functions
$g(N_i)$, $c(N_i)$ and $J(\Rr(N_i))$. The proposed algorithm is agnostic with respect to these functions.
Therefore, the user can specify any formulation of them by ensuring that the following criteria are satisfied~\cite{schmid2020efficient}:
\begin{enumerate}
	\item $g(N_i)$ should be a function that depends on the trajectory end position only ($g(\rr^i_5, \phi^i_3)$),
	\item all node gains should be mutually independent,
	\item $c(N_i)$ is required to be an intrinsic property of the trajectory ($c(\CP^{\rr}_i, \CP^{\phi}_i, \delta_i)$).
\end{enumerate}

\subsection{Tree Update}\label{sec:TREE-UPDATE}
The tree, initially composed by just one root node, is iteratively expanded by randomly sampling viewpoints inside a sphere
centred on the current \textit{best node} ($N_{\text{best}}$), namely the one among all tree nodes that maximise the utility
function $J(\cdot)$. In particular, the sphere is centred exactly on the last control point of $\CP^{\rr}_{\text{best}}$,
i.e. $\rr^{\text{best}}_5$, while its radius ($r_{\text{sp}}$) is a user chosen value defined as a parameter for the algorithm.
The sampled viewpoint is retained only if it belongs to a known and free part of the environment under exploration and,
at the same time, it is far enough from the mapped obstacles.
Such viewpoint is considered as the last control point of the next trajectory segment ($\rr^{i}_5$). Moreover, due to Condition~\eqref{eq:FIRST-COND},
also the control point $\rr^{i}_0$ is already defined to ensure position continuity.
As regards the heading trajectory, the first control point ($\phi^i_0$) is established through Condition~\eqref{eq:FOURTH-COND},
while the last one ($\phi^i_3$) is chosen as the value that maximise the potential information gain $g(\rr_5^i, \phi)$, namely
\begin{equation*}
	\phi_i^3 = \arg \max_{\phi} \ g(\rr_5^i, \phi),
\end{equation*}
in a similar way as done in~\cite{selin2019efficient}.
The choice of the remaining points ($\CP^{\rr}_i\left[ 1:4 \right]$, $\CP^{\phi}_i \left[ 1:2 \right]$) and the trajectory duration ($\delta_i$)
is performed concurrently. In particular, the interval of admissible trajectory duration $\Delta = [\delta_{\text{min}}, \delta_{\text{max}}]$
is uniformly discretised as
\begin{equation*}
	\Delta_{\text{d}} = \left\{ \delta_{\text{min}},\  \delta_{\text{min}} + \Delta_{\delta}, \  \delta_{\text{min}} + 2\Delta_{\delta},\  \dots, \ \delta_{\text{max}} \right\},
\end{equation*}
with $\Delta_{\delta} = \frac{\delta_{\text{max}} - \delta_{\text{min}}}{r}$, leading to $r+1$ possible time intervals.
For any $\delta \in \Delta_{\text{d}}$, the control points $\rr^{i}_1$ and $\rr^{i}_2$ are computed exploiting Condition~\eqref{eq:SECOND-COND}
and Condition~\eqref{eq:THIRD-COND}. If the obtained points do not satisfy Proposition~\ref{th:COLLISION-FREE}
the current $\delta$ is discarded, otherwise also the points $\rr^{i}_3$ and $\rr^{i}_4$, as well as $\phi^i_1$
(Condition~\eqref{eq:FIFTH-COND}) and $\phi^i_2$ are computed. 
Note that the quantities $\rr^{i}_3$, $\rr^{i}_4$ and $\phi^i_2$ are not constrained by any conditions~\eqref{eq:FIRST-COND}--\eqref{eq:FIFTH-COND}, thus these points are computed by optimising the induced cost $c(\CP_i^{\rr}, \CP_i^{\phi}, \delta_i)$.
In the same way as before, if the computed points violate Proposition~\ref{th:COLLISION-FREE},
or the induced velocity or acceleration exceed the dynamic bounds, the current $\delta$ is discarded.
Once all $\delta \in \Delta_{\text{d}}$ have been considered, the one leading to the optimal value of $c_i$ is selected with
the corresponding computed control points and the node is added to the tree.
The tree growth continues until it became impossible to find a new node with higher information gain $g(\cdot)$ and the number
of sampled nodes go beyond a given threshold ($n_{\text{max}}$). Once the tree expansion is terminated, the branch leading to
the best node is extracted and only the first node of such branch ($N_\text{opt}$) is executed.

It may happen that the tree growth procedure takes too much time, or it may result impossible to find a
valid candidate as next trajectory segment. In order to handle these issues, at each iteration two trajectory segments are computed.
The first one corresponds to the execution of the best branch, while the second one is a \emph{safe} trajectory,
linked via Constraints~\eqref{eq:FIRST-COND}--\eqref{eq:FIFTH-COND} to the first committed segment.
The \emph{safe} trajectory constrains the final velocity to be zero, as well as the final acceleration,
and it is executed every time the algorithm fails in planning a new node.

\section{Cost Formulation}\label{sec:COST-FORMULATION}
The algorithm presented in Section~\ref{sec:PROPOSED-APPROACH} is used to plan spatial trajectories by maximising
the total utility function $J(\cdot)$. As a consequence, since this function combines both node gain and cost,
the choice of functions $g(\cdot)$ and $c(\cdot)$, as well as $J(\cdot)$ itself, is crucial for the success of the exploration procedure and of its performance.
\subsection{Reconstruction Gain $g(\rr, \phi)$}\label{sec:RECONSTRUCTION-GAIN}
The reconstruction gain is defined as the amount of space that can be discovered if the agent is located in the considered position ($\rr$)
and oriented with a given heading angle ($\phi$).
The function $g(\cdot)$ can be computed by casting rays outward from the sensor and summing up all the unmapped volume elements that the ray crosses.
Although there exist very efficient procedures useful to compute $g(\cdot)$, such as sparse ray-casting~\cite{selin2019efficient},
its explicit evaluation is still the bottleneck for most of the exploration algorithms proposed within the literature.
The work~\cite{selin2019efficient}, motivated by the continuous nature of the reconstruction gain over its domain,
tries to overtake this problem by modelling it as a realisation of a Gaussian Process (GP)~\cite{rasmussen2003gaussian}.
The idea is to infer, when possible, the gain value using previously sampled data, avoiding its explicit computation at each exploration iteration.
This approach has the limitation that when the process returns a \textit{poor} (in terms of resulting variance) estimation,
the gain must be explicitly re-computed, leading to a higher overhead due to the double computation.
In this work we show that it is possible to completely avoid the gain explicit computation in real-time, during the planning procedure,
as it can be left as a background thread.
Unlike previous approaches, we propose to evaluate $g(\cdot)$ exclusively through Gaussian Process inference.
The motivating assumption is that in previously unexplored areas the reconstruction gain evaluates as the sensor FoV volume.
Therefore we impose a GP prior $g(\rr) \sim \mathcal{GP}(V_{\text{fov}}, k(\rr, \rr', \tau))$ consisting of a constant mean function, equivalent to the sensor FoV volume, and a \textit{squared-exponential} kernel
\begin{equation*}
	k(\rr, \rr', \tau) = \exp \left( - \frac{\| \rr - \rr' \|^2_2 }{2\tau^2} \right),
\end{equation*}
where $\tau$ is a hyper-parameter known as \textit{characteristic length-scale},
iteratively estimated by minimising the associated log-likelihood function~\cite{rasmussen2003gaussian}.
The proposed approach alternates between gain prediction, using the currently sampled data and the current estimation of the hyper-parameter,
and correction, where $\tau$ is estimated by minimising the log-likelihood over the data.

\subsection{Trajectory Cost $c(N_i)$}\label{sec:TRAJECTORY-COST}
In autonomous exploration applications, where classical RRT algorithms are employed, the trajectory cost is usually associated with node
distance~\cite{selin2019efficient, bircher2016receding}, or execution time~\cite{schmid2020efficient}. The former penalises long trajectories,
while the latter pushes the agent near its dynamical limits in order to execute the task as fast as possible. Recent studies about frontier
exploration~\cite{cieslewski2017rapid} have shown great results in terms of execution time and traveled distance. In these works, viewpoints
are selected considering minimal variations in velocity.
Encouraged by the success of these algorithms and keeping in mind the necessity to end the exploration as fast as possible,
we propose a trajectory cost that weight execution time and total control effort. The overall trajectory cost is formalised as
\begin{equation*}
	c(\CP^{\rr}_{i}, \CP^{\phi}_{i}, \delta_{i}) = \mu_1 \delta_i + \mu_2 c_{\rr}(\CP^{\rr}_{i}, \delta_{i}) + \mu_3 c_{\phi}(\CP^{\phi}_{i}, \delta_i),
\end{equation*}
where $\mu_{1:3}$ are tuning parameters, while $c_{\rr}(\cdot)$ and $c_{\phi}(\cdot)$ take the following form
\begin{align}
	\label{eq:TRAJ-COST-R}
	c_{\rr}(\CP^{\rr}_i, \delta_i) & = \int_0^{\delta_i}  \norm{\frac{d^k \rr^i(\frac{\tau}{\delta_i})}{d\tau^k}}^2 d\tau, \\
	\label{eq:TRAJ-COST-PHI}
	c_{\phi}(\CP^{\phi}_i, \delta_i) & = \int_0^{\delta_i}  \norm{\frac{d^p \phi^i(\frac{\tau}{\delta_i})}{d\tau^p}}^2 d\tau.
\end{align}
In this particular case we selected $k = 2$ and $p = 1$, leading to trajectories with minimal accelerations and angular velocities.
Minimising the angular velocity has several benefits in terms of mapping reconstruction accuracy, due to the fact that
the captured data present low blur effect, especially when working with rgb cameras.
Note that the three components of $\rr(\cdot) = [r_x(\cdot), r_y(\cdot), r_z(\cdot)]$ are decoupled inside the cost function,
thus Equation~\eqref{eq:TRAJ-COST-R} can be rewritten as
\begin{equation}
	\label{eq:TRAJ-COST-R-DECOUPLED}
	c_{\rr}(\CP^{\rr}_i, \delta_i) = \sum_{j}^{x,y,z} \int_0^{\delta_i}  \frac{d^k r^i_{j}(\frac{\tau}{\delta_i})}{d\tau^k}^2 d\tau.
\end{equation}
The Bernstein basis parameterisation is closed with respect to operations of derivative, power elevation and integral~\cite{farouki2012bernstein},
thus the Equations~\eqref{eq:TRAJ-COST-PHI} and~\eqref{eq:TRAJ-COST-R-DECOUPLED} can be evaluated in closed form just acting on the trajectory
control points in the following way
\begin{align*}
    c_{\rr}(\CP^{\rr}_i, \delta_i) & =
	\begin{bmatrix}
		\rr^i_0 & \cdots & \rr^i_{n_{\rr}}
	\end{bmatrix}
	B_{\rr}(\delta_{i})
	\begin{bmatrix}
		\rr^i_0 & \cdots & \rr^i_{n_{\rr}}
	\end{bmatrix}^T, \\
	c_{\phi}(\CP^{\phi}_i, \delta_i) & =
	\begin{bmatrix}
		\phi^i_0 & \cdots & \phi^i_{n_{\phi}}
	\end{bmatrix}
	B_{\phi}(\delta_i)
	\begin{bmatrix}
		\phi^i_0 & \cdots & \phi^i_{n_{\phi}}
	\end{bmatrix}^T
\end{align*}
where $B_{\rr}(\delta_{i})$ and $B_{\phi}(\delta_{i})$ are the matrix form of the Bézier curves $c_{\rr}(\cdot)$ and $c_{\phi}(\cdot)$~\cite{qin2000general}.
We take advantage of this property during the planning stage, when selecting the remaining free points $\rr_3^i$, $\rr_4^i$ and $\phi_2^i$.
These are computed solving the following optimisation problem
\begin{equation*}
	\min_{\rr^i_3, \rr^i_4, \phi^i_2} 
	\begin{bmatrix}
		\rr^i_0 \\ \vdots \\ \rr^i_{n_{\rr}} \\ \phi^i_0 \\ \vdots \\ \phi^i_{n_{\phi}}
	\end{bmatrix}^T
	\begin{bmatrix}
		\mu_2 B_{\rr}(\delta_{i}) \\
		\mu_3 B_{\phi}(\delta_i)
	\end{bmatrix}
	\begin{bmatrix}
		\rr^i_0 \\ \vdots \\ \rr^i_{n_{\rr}} \\ \phi^i_0 \\ \vdots \\ \phi^i_{n_{\phi}}
	\end{bmatrix},
\end{equation*}
which results in an unconstrained QP problem, solvable by equalising the gradient to zero in a similar way as done in~\cite{richter2016polynomial}.

\begin{figure*}[t]
	\centering
	\begin{subfigure}[t]{0.32\textwidth}	
		\centering
		\includegraphics[trim={3cm 1.5cm 3cm 1cm}, clip = true, width = 1.1\textwidth]{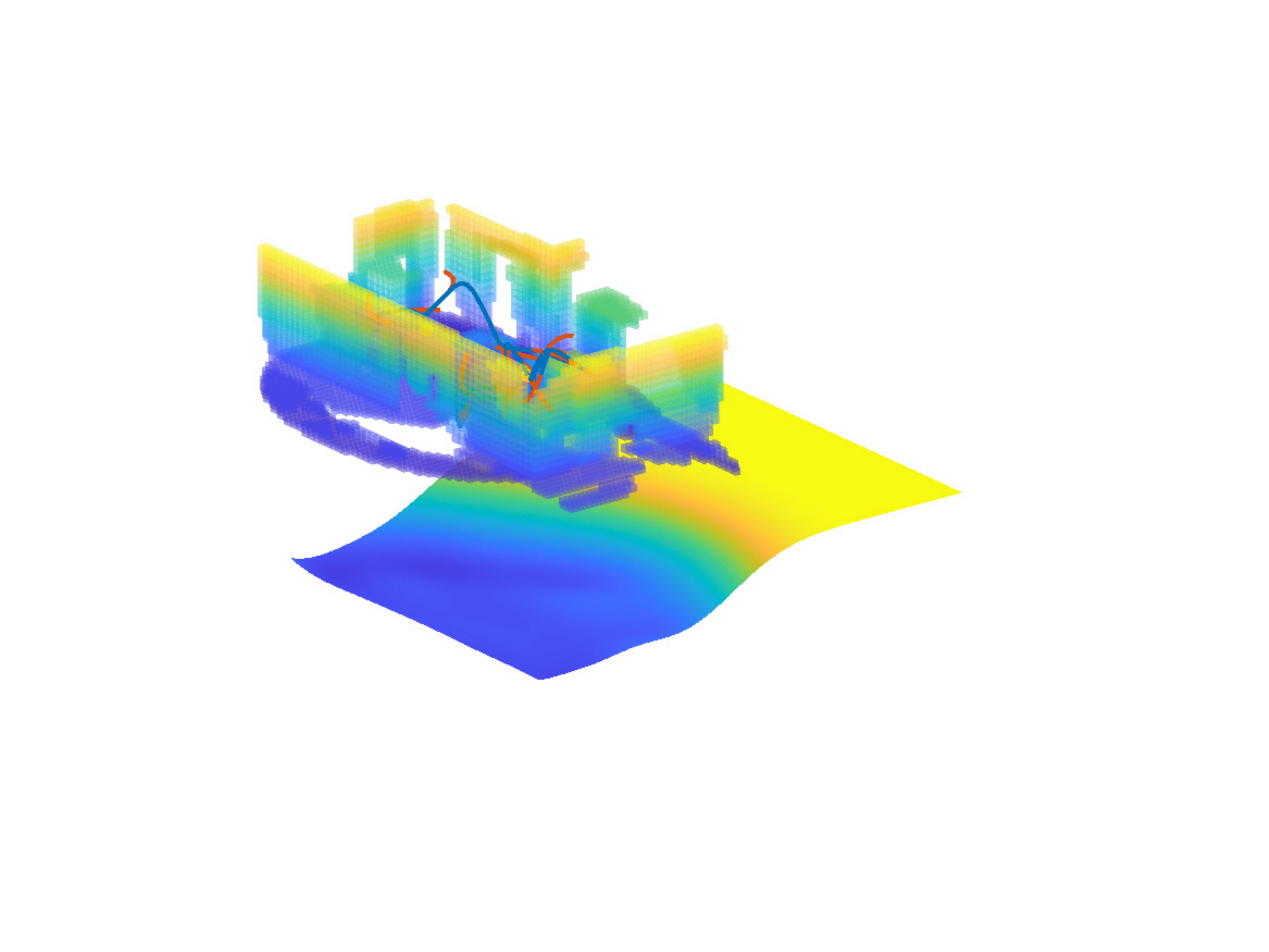}
		\vspace{-1.4cm}
		\caption{Exploration state at $100$ seconds.}\label{fig:EXP-SIM-TEST-A}
	\end{subfigure}
	\begin{subfigure}[t]{0.32\textwidth}
		\centering
		\includegraphics[trim={5cm 1.5cm 3cm 1cm}, clip = true, width = 1\textwidth]{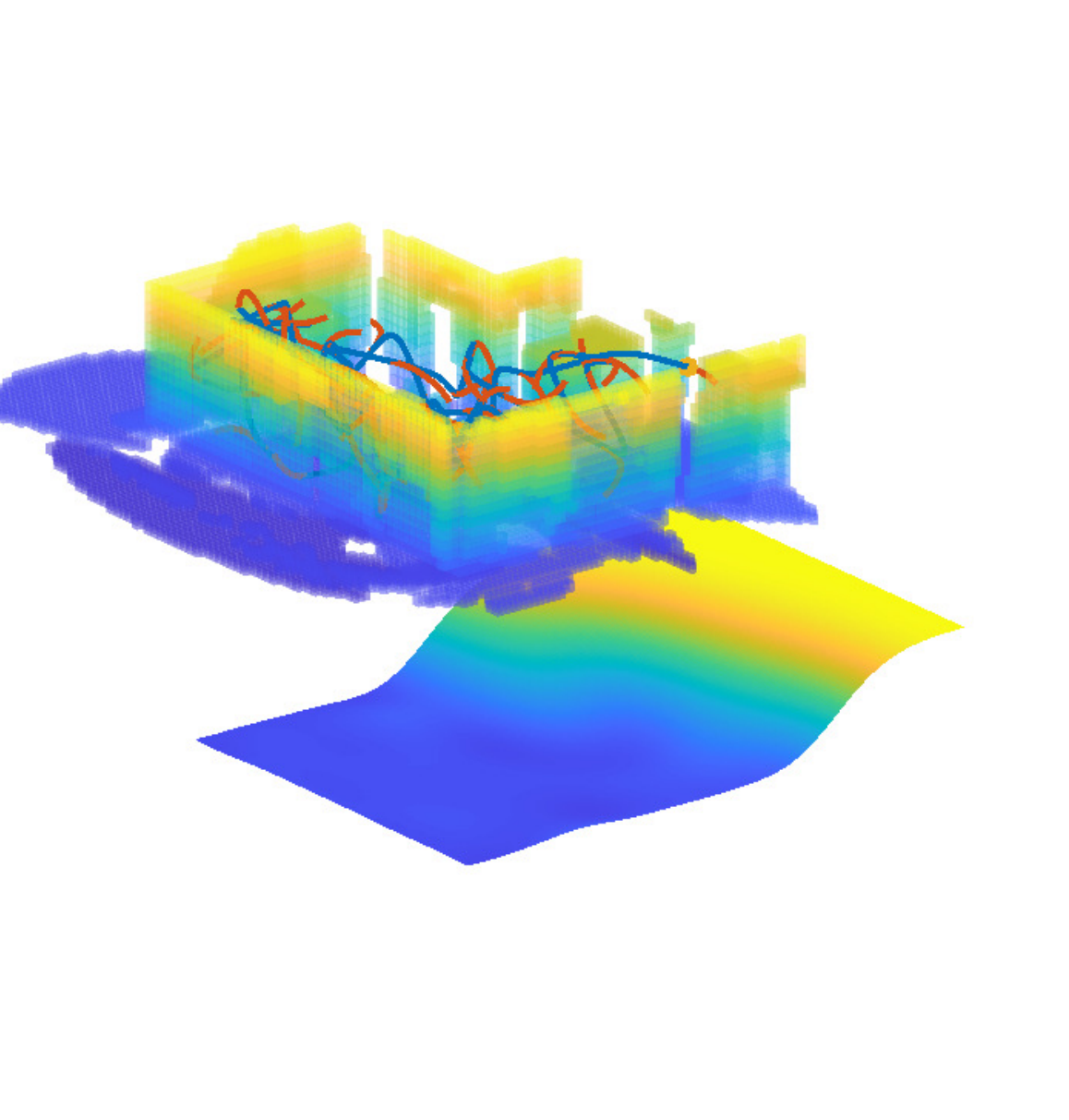}
		\vspace{-1.4cm}
		\caption{Exploration state at $200$ seconds.}\label{fig:EXP-SIM-TEST-B}
	\end{subfigure}
	\begin{subfigure}[t]{0.32\textwidth}
		\centering
		\includegraphics[trim={5cm 1.5cm 3cm 1cm}, clip = true, width = 1\textwidth]{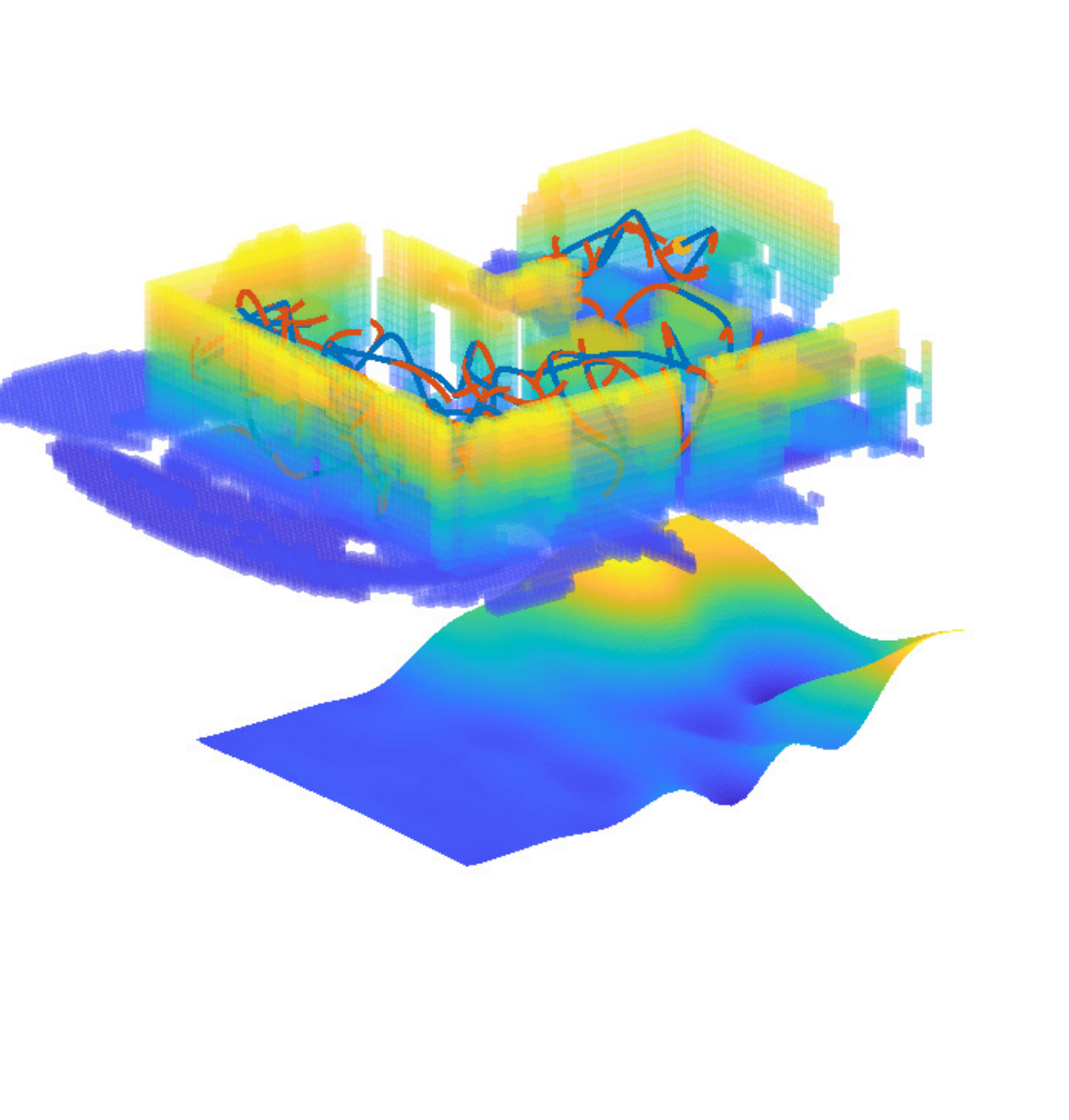}
		\vspace{-1.4cm}
		\caption{Exploration state at $300$ seconds.}\label{fig:EXP-SIM-TEST-C}
	\end{subfigure}
	\vspace{-0.4cm}
	\caption{Results of the simulation tests. The exploration algorithm runs over a map of $20\times10\times3$
	meters and was able to complete the exploration after only $400$ seconds.}\label{fig:EXP-SIM-TEST}
\end{figure*}

\subsection{Total Utility $J(\mathcal{R}(N_i))$}\label{sec:TOTAL-UTILITY}
The total utility function is responsible to merge gains and costs of the tree nodes in only one utility value,
which is used to select the best node during exploration.
Along this work we borrow the idea of~\cite{schmid2020efficient}, that
proposes a total utility function based on the notion of efficiency:
\begin{equation*}
	J(\mathcal{R}(N_i)) = \frac{\sum_{N_l \in \mathcal{R}(N_i)} g_l}{\sum_{N_l \in \mathcal{R}(N_i)} c_l}.
\end{equation*}

\section{Implementation Details}\label{sec:IMPLEMENTATION-DETAILS}
\begin{figure}[!t]
	\centering
	\includegraphics[scale=.7]{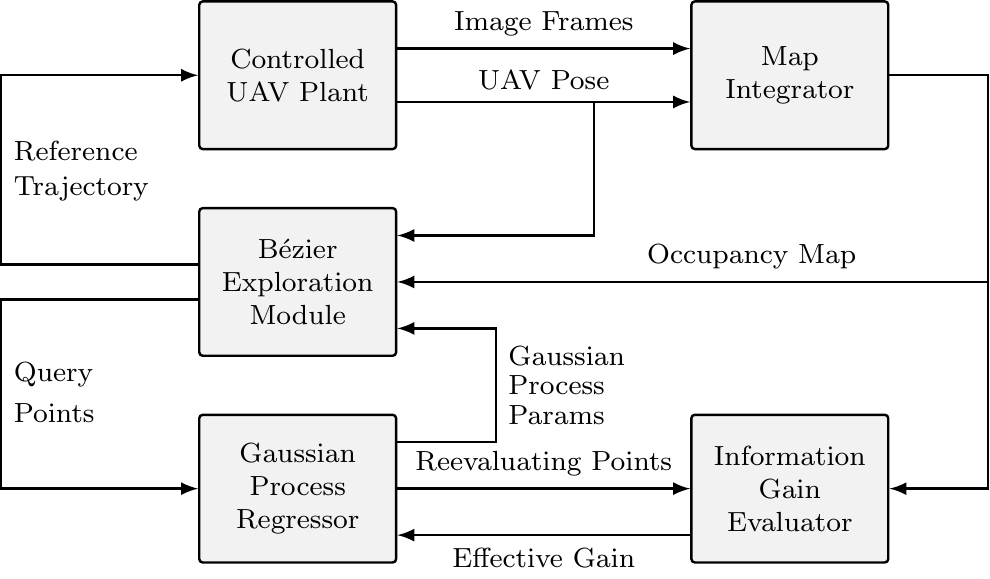}
	\caption{The overall scheme of the proposed exploration system.}\label{fig:ALG-SCHEME}
\end{figure}
The overall structure of the proposed framework is shown in \figref{fig:ALG-SCHEME}.
The planning framework is built on top of a reliable UAV control scheme and an occupancy map
integrator whose construction is out of the scope of this work.
The solution relies on three different threads running in parallel.
\subsubsection{Bézier Exploration Module}
acts as the exploration supervisor and is in charge of computing the reference trajectories to be executed from the UAV platform.
The exploration module continuously grows and maintains the trajectories tree via random sampling
and by reevaluating each sampled non-executed node at each new iteration.
Since the tree is executed in a receding-horizon fashion, every time a new trajectory is commissioned, a bunch of previously
planned trajectories may become infeasible due to continuity issues. To handle this problem an activation flag is added as a node property.
Note that at each sampling, only the active branches are taken into consideration.
Furthermore, the explorer node also keeps care about possible \textit{deadline violations}.
In these cases the executed safe trajectory is added to the tree and rewired to all nearby active nodes.
This prevent us to lose previous sampled possible promising trajectories.
\subsubsection{Gaussian Process Regressor}
receives all sampled points from the exploration module and implements a policy to allow the cache of only the most informative ones.
In particular, a new point is retained only if it belongs to a new and not explored area.
The implementation of an R-tree allows for fast point insertion and retrieval, moreover it eases the insertion condition check.
In order to be consistent with the exploration task, and the evolution of the known map, all cached points are reevaluated periodically via explicit gain computation.
Such a module is in charge to train the Gaussian process parameters used by the explorer.
\subsubsection{Information Gain Evaluator}
receives the evaluating point and computes explicitly the information gain via sparse ray-casting as in~\cite{selin2019efficient}.

The architectural subdivision of the implemented algorithm in three different threads allows fast computing high informative trajectories
without the explicit gain computation bottleneck.
The whole algorithm has been implemented as a ROS network and paired with the PX4 autopilot both for the software-in-the-loop
simulations and the real-world tests. As a map representation we use OctoMap~\cite{hornung2013octomap}.
{
\renewcommand{\arraystretch}{1.35}
\begin{table}[b!]
    \centering
    \begin{tabular}{|c|c|c|c|}
        \hline
        Max Vel. & 1.5 m/s & Max Acc. & 1.5 m/s$^2$ \\
        \hline
        Sampled Nodes & 40 & Max Length & 3 m \\
        \hline
        Min Range & 0.3 m & Max Range & 5.0 m \\
        \hline
        Camera FoV & $115\times60$ deg & Map Res. & 0.2 m \\
        \hline
        $\mu_1$ $\mu_2$ $\mu_3$ & 0.5 0.1 0.1 & Time Res. & 0.5 s \\
        \hline
        Min Time & $1$ s& Max Time & $5$ s \\
        \hline
    \end{tabular}
    \caption{Parameters used in simulations.}\label{tab:SIMULATION-PARAMETERS}
\end{table}
\begin{table}[b!]
    \centering
    \begin{tabular}{|c|c|c|c|}
        \hline
        Max Vel. & 0.5 m/s & Max Acc. & 0.5 m/s$^2$ \\
        \hline
        Sampled Nodes & 20 & Max Length & 3 m\\
        \hline
        Min Range & 0.3 m & Max Range & 3.0 m \\
        \hline
        Camera FoV & $87\times58$ deg & Map Res. & 0.2 m \\
        \hline
        $\mu_1$ $\mu_2$ $\mu_3$ & 0.5 0.1 0.1 & Time Res. & 0.5 s \\
        \hline
        Min Time & $1$ s & Max Time & $5$ s \\
        \hline
    \end{tabular}
    \caption{Parameters used in the real-world experiments.}\label{tab:REAL-PARAMETERS}
\end{table}
}
\section{Experimental Evaluation}\label{sec:EXPERIMENTAL-EVALUATION}
The proposed approach has been evaluated via Gazebo-based simulations, exploiting the environment RotorS~\cite{furrer2016rotors} along with
the provided 3DR Iris quadrotor model, endowed of a depth sensor.
The algorithm performance have been also qualitatively evaluated in real-world scenario tests. In all tests the agent starts
in the origin with zero yaw angle. The agent performs an initial action of rotating $360$ degrees around the initial hovering
point in order to be sure to start the exploration with some initial information at hand.
\subsection{Simulation Tests}
\begin{figure*}[t!]
	\centering
	\begin{subfigure}[t]{0.32\textwidth}	
		\centering
		\includegraphics[trim={5cm 1.5cm 2cm 2.2cm}, clip = true, width = 1\textwidth]{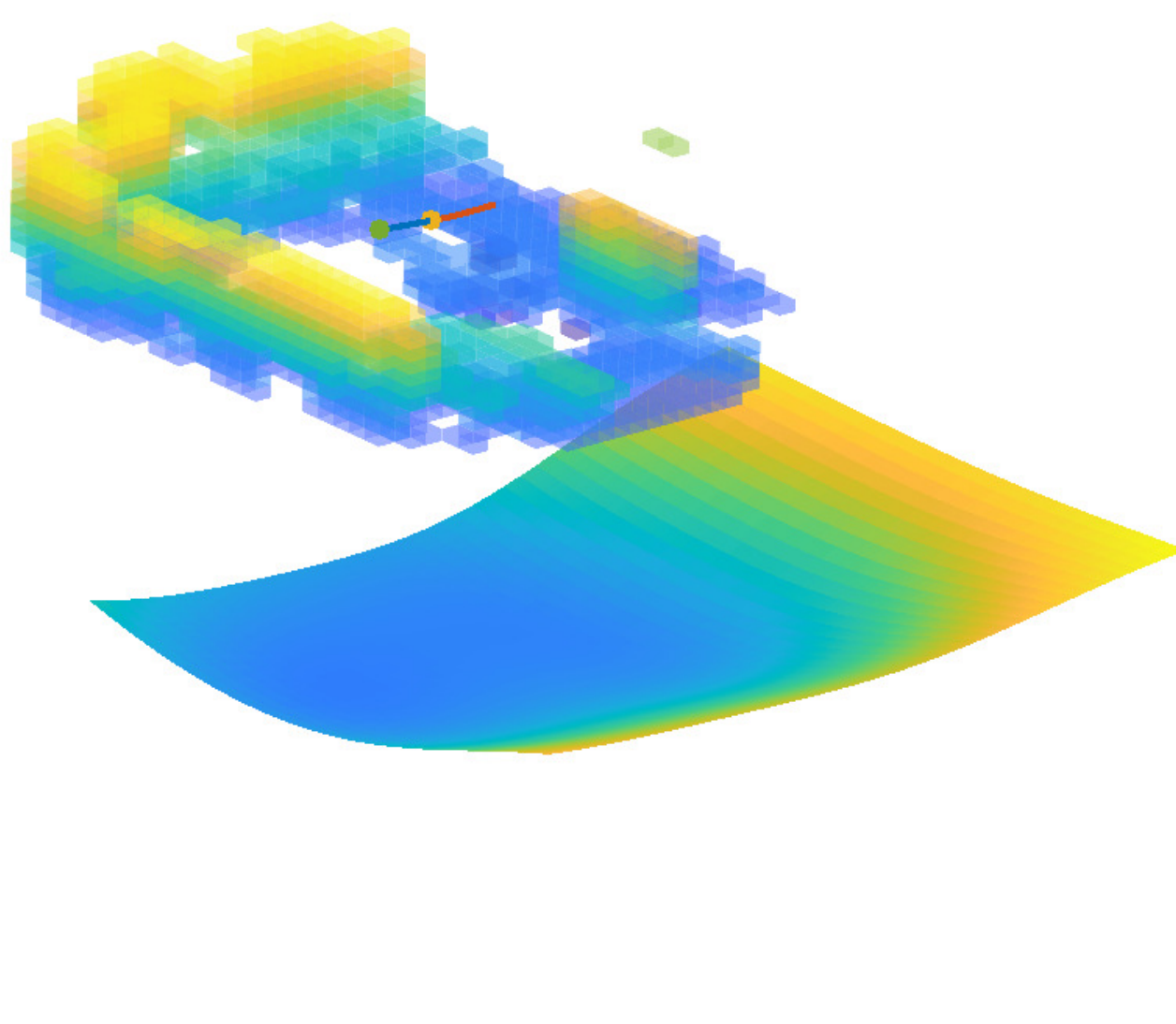}
		\vspace{-1.6cm}
		\caption{Exploration state at $20$ seconds.}\label{fig:EXP-REAL-TEST-A}
	\end{subfigure}
	\begin{subfigure}[t]{0.32\textwidth}
		\centering
		\includegraphics[trim={5cm 1.5cm 3cm 2.2cm}, clip = true, width = 0.9\textwidth]{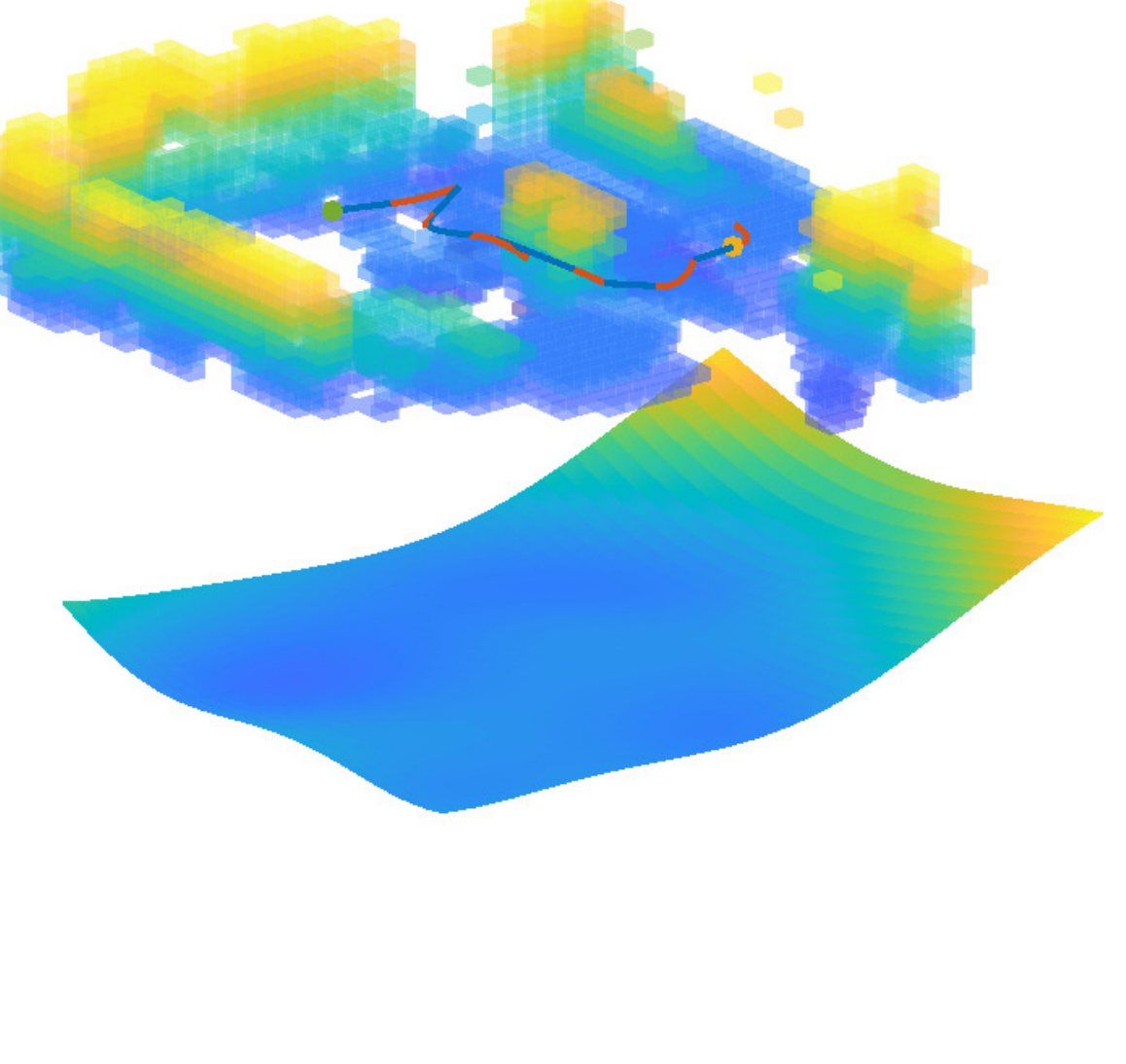}
		\vspace{-1.2cm}
		\caption{Exploration state at $57$ seconds.}\label{fig:EXP-REAL-TEST-B}
	\end{subfigure}
	\begin{subfigure}[t]{0.32\textwidth}
		\centering
		\includegraphics[trim={5cm 1.5cm 3cm 2.2cm}, clip = true, width = 1\textwidth]{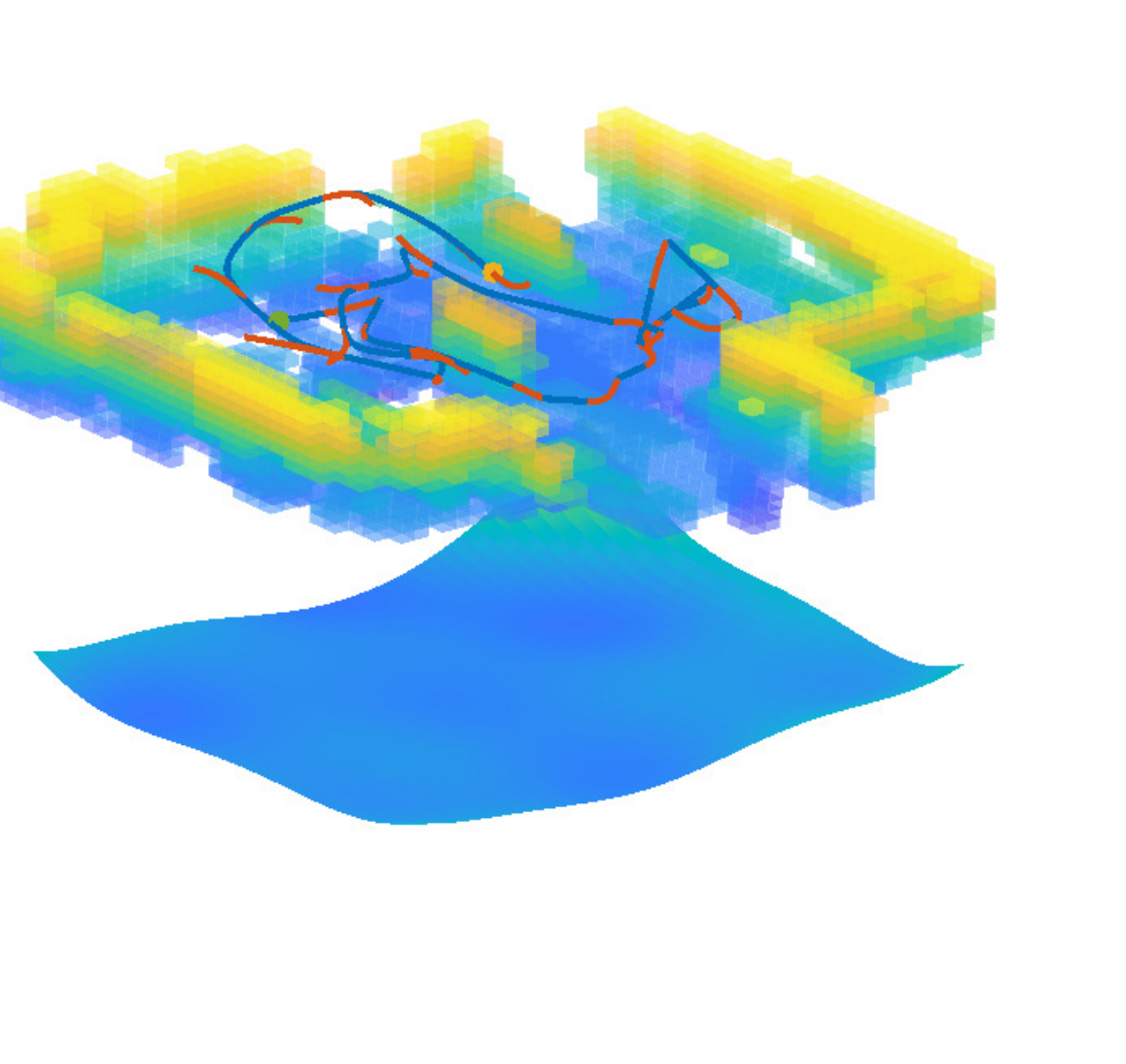}
		\vspace{-1.6cm}
		\caption{Exploration state at $167$ seconds.}\label{fig:EXP-REAL-TEST-C}
	\end{subfigure}
	\vspace{-0.6cm}
	\caption{Results of the real-world exploration test.
			 The exploration algorithm was run using only integrated onboard sensors and computational capabilities.}\label{fig:EXP-REAL-TEST}
\end{figure*} 
\begin{figure}[!t]
	\centering
	\includegraphics[scale=.25]{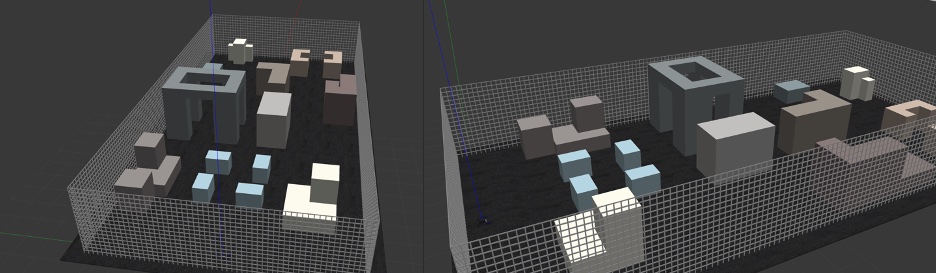}
	\caption{Simulation environment.}\label{fig:SIMULATING-SCENARIO}
\end{figure}
\begin{figure}[!t]
	\centering
	\includegraphics[scale=.15]{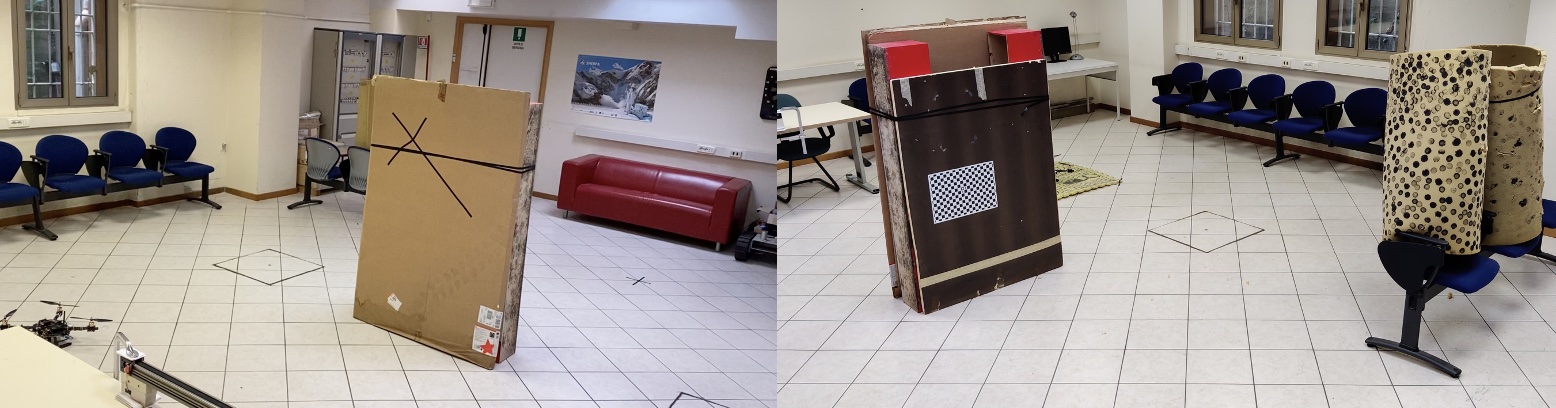}
	\caption{Real-world environment.}\label{fig:REAL-SCENARIO}
\end{figure}
The parameters used for the simulations are reported in Table~\ref{tab:SIMULATION-PARAMETERS}.~\figref{fig:EXP-SIM-TEST}
shows the obtained simulation results when agent is required to map an $20\times10\times3$
urban canyon (see~\figref{fig:SIMULATING-SCENARIO}). In particular, in~\figref{fig:EXP-SIM-TEST}, the blue
lines represent the reference trajectory, the red ones are the planned safe motions (both executed and non-executed),
while the bottom surface represents the current Gaussian process state. It can be noticed that,
the Gaussian process is constantly kept updated with the current map information and it results to be consistent,
at each time instant, with the exploration task.
In order to evaluate the performances against the state-of-the-art solutions, the proposed algorithm
has been compared with the \textit{Autonomous Exploration Planner} (AEP) described in~\cite{selin2019efficient}.~\figref{fig:COMPARED-ALGOROTHMS} compares the amount of explored volume over time by both the approaches.
The blue line represents the average of explored area obtained deploying our approach over $10$ experiments,
with the associated standard deviation represented in shaded blue. Conversely, the line and shade red reports the
results obtained via AEP, with the global exploration module disabled, on the same number of experiments.
It can be noticed that both algorithms achieve comparable results at the beginning of the exploration,
where most of the volume needs to be explored, then our solution tends to get higher exploration rate,
thanks to the ability to fast plan the next trajectory. Moreover, it is worth noting that our solution provides more consistency between different tests, as the variance is narrower with respect to the AEP, thus guaranteeing better repeatability of the experiment and a mitigation of the worst case scenarios.
~\figref{fig:TRAVELLED-DISTANCE} depicts
the overall travelled distance on same experiments.
Since our solution plans trajectories by never stopping the UAV motion, this leads to an overall
travelled distance $2$ times higher than the AEP solution.
\begin{figure}[!t]
	\centering
	\includegraphics[trim={0cm 0cm 0cm 1cm}, clip = true, scale=.4]{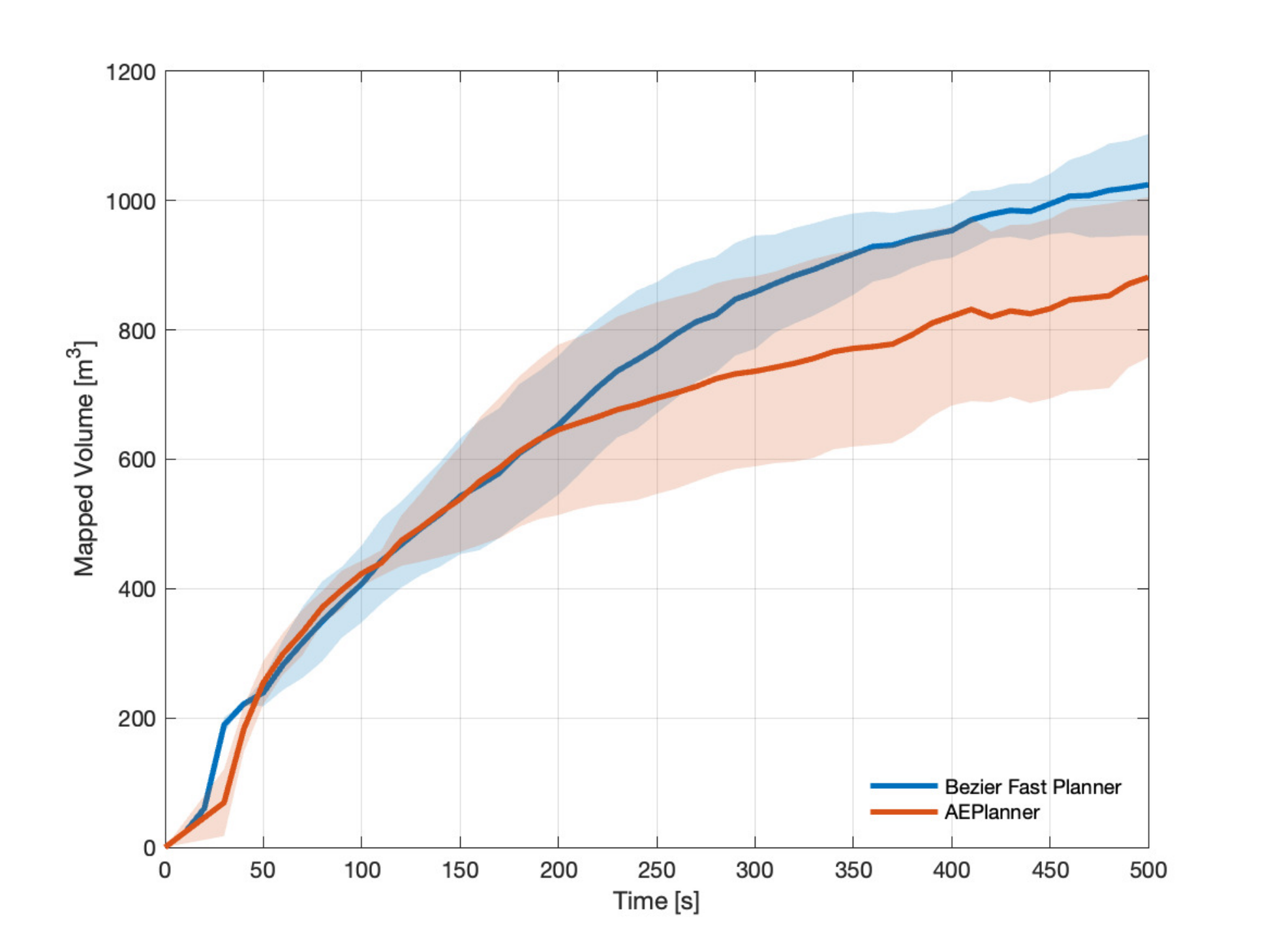}
	\vspace{-0.2cm}
	\caption{Exploration progress for the urban $20 \times 10 \times 3$ canyon. Mean and standard deviation over $10$ experiments are shown. Notice that due to the employing of pierced nets as maps borders makes the overall explored volume higher than the real volume.}\label{fig:COMPARED-ALGOROTHMS}
\end{figure}
\begin{figure}[!t]
    \vspace{-0.5cm}
	\centering
	\includegraphics[scale=.4]{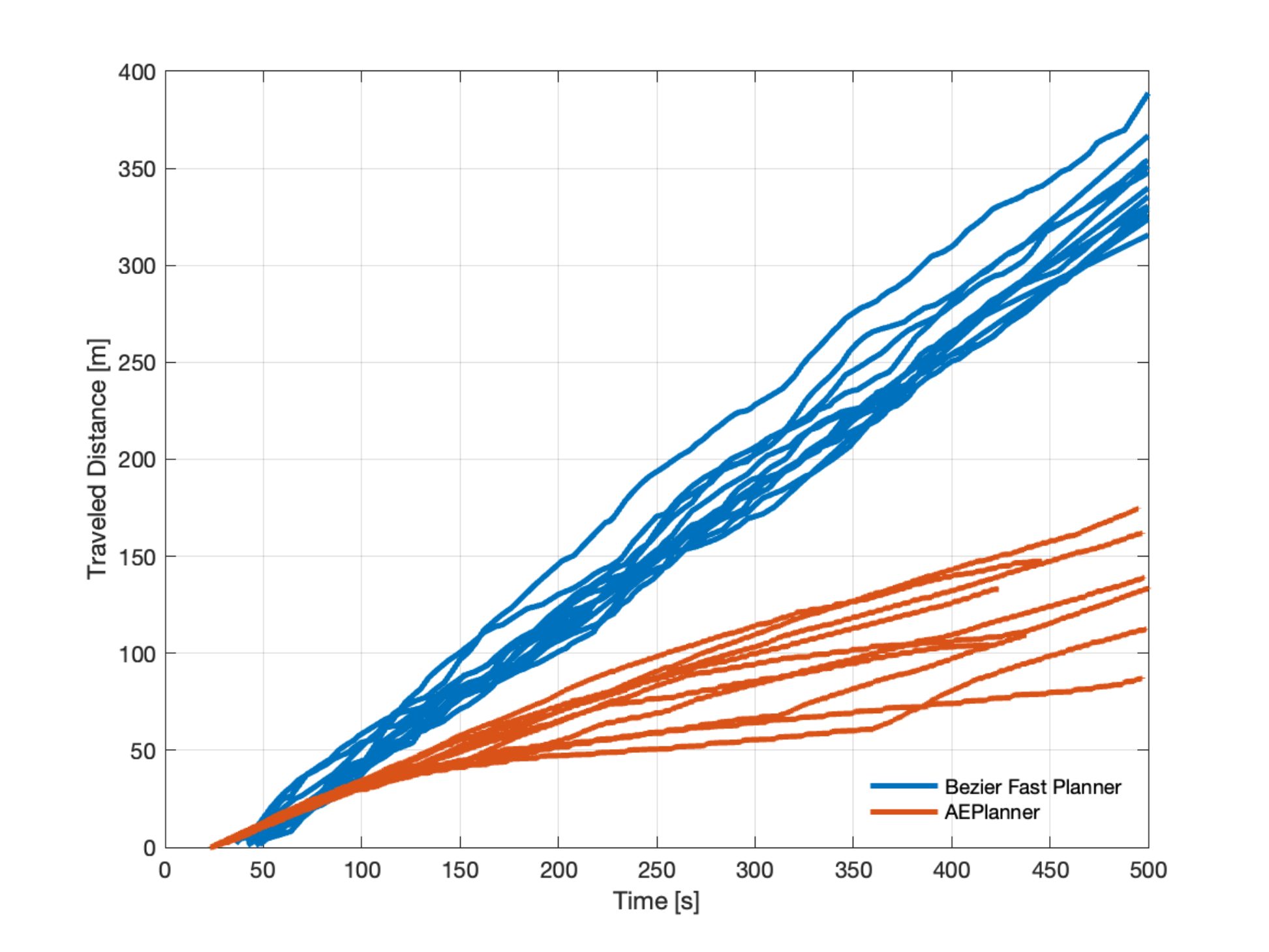}
	\vspace{-0.2cm}
	\caption{Overall traveled distance in the urban canyon. The traveled distances over $10$ experiments are shown.}\label{fig:TRAVELLED-DISTANCE}
\end{figure}
\subsection{Real-World Tests}\label{sec:REAL-TESTS}
The solution has been tested using a real UAV inside an indoor scenario using our office spaces.~\figref{fig:REAL-SCENARIO}
shows the scenario with a couple of high obstacles placed at the centre of the area,
while the obtained results are depicted in~\figref{fig:EXP-REAL-TEST}. The used parameters are reported in
Table~\ref{tab:REAL-PARAMETERS}. The available area was $9\times6\times2.5$ meters and it was successfully
mapped in $170$ seconds, the maximum camera range was saturated at 3 meters in order to stress navigation
trajectories around obstacles. The used UAV was powered by the PX4 autopilot and endowed with a
depth Intel RealSense D455 camera, mounted frontally. Visual odometry, used to localize the UAV in the indoor
scenario, was provided by an Intel Realsense T265 tracking camera. Video of the experiment can be found at: \href{https://youtu.be/4respaTDGsg}{youtu.be/4respaTDGsgy} 
\section{Conclusions}\label{sec:CONCLUSIONS}
In this work, we presented a novel approach to the problem of rapid exploration in three-dimensional unknown environments using UAVs.
The proposed solution speeds-up local exploration by sequentially building a tree of high informative trajectories.
The Bèzier curve parameterisation guarantees fast checking for collisions, while the Gaussian process inference permits
fast gain retrieval. The separate threads implementation was the key to avoid the gain computation bottleneck and
to allow the algorithm to run without stopping the exploration agent. Simulations show that the combination of path
planning and time allocation succeeds in sensibly incrementing the exploration rate with respect to state-of-the-art solutions.
As future work we propose to exploit more the Gaussian process properties by incorporating, in the exploration procedure,
also the gradient information. Furthermore, we will deeply investigate
new solutions in the context of global exploration, that can be efficiently coupled with this local exploration solution.

\end{document}